\documentclass[letterpaper, 10 pt, conference]{ieeeconf}

\IEEEoverridecommandlockouts
\usepackage{url} 
\usepackage{algorithm}
\usepackage{algorithmic}
\usepackage{amsmath,amssymb,amsfonts}
\usepackage[font=small,labelfont=bf]{caption}
\usepackage{epsfig,graphics}
\usepackage{xcolor}
\usepackage{cite}

\usepackage{pgf}
\usepackage{tikz}
\usetikzlibrary{arrows,automata}

\graphicspath{{./fig/}}

\usepackage{theorem}
\newtheorem{lemma}{Lemma}
\newtheorem{remark}{Remark}

\newtheorem{assumption}{Assumption}
\newtheorem{corollary}{Corollary}
\newtheorem{definition}{Definition}

\newtheorem{problem}{Problem}
\newtheorem{example}{Example}

\newcommand{\T}{\mathcal{T}} 
\newcommand{\A}{\mathcal{A}} 
\newcommand{\B}{\mathcal{B}} 

\renewcommand{\P}{\mathcal{P}} 
\newcommand{\init}{\mathit{init}}
\newcommand{\currs}{\mathfrak{s}}
\newcommand{\currq}{\mathfrak{q}}

\newcommand{\AP}{\Pi} 
\newcommand{\APs}{\mathbf{\Pi}}
\newcommand{\Lang}{\mathcal{L}} 
\newcommand{\Set}{\mathsf{S}} 
\newcommand{\Spec}{\mathbf{\Phi}}

\newcommand{\Nat}{\mathbb{N}} 

\newcommand{\Next}{\mathsf{X}}
\newcommand{\Until}{\mathsf{U}}
\newcommand{\Always}{\mathsf{G}}
\newcommand{\Event}{\mathsf{F}}
\newcommand{\true}{\mathit{true}}

\renewcommand{\epsilon}{\varepsilon}

\newcommand{\prop}{\pi}

\newcommand{\ie}{{i.e., }}
\newcommand{\eg}{{e.g., }}

\newcommand{\h}{h}
\renewcommand{\H}{H}

\renewcommand{\alph}{\Sigma}
\newcommand{\Alpha}{\mathbf{\Sigma}}
\renewcommand{\mod}{\mathrm{\, mod \, }}
\newcommand{\suc}{\mathit{succ}}

\newcommand{\new}[1]{ #1}

\begin{document}

\title{A Receding Horizon Approach to Multi-Agent Planning from Local LTL Specifications}
\author{Jana T\r{u}mov\'a and Dimos V. Dimarogonas \thanks{The authors are with the ACCESS Linnaeus Center, School of Electrical
Engineering, KTH Royal Institute of Technology, SE-100 44, Stockholm,
Sweden and with the KTH Centre for
Autonomous Systems. \texttt{tumova, dimos@kth.se}. This work was supported by the EU STREP RECONFIG.}}

\maketitle
\thispagestyle{empty}
\pagestyle{empty}

\begin{abstract}
We study the problem of control synthesis for multi-agent systems, to achieve complex, high-level, long-term goals that are assigned to each agent individually. As the agents might not be capable of satisfying their respective goals by themselves, requests for other agents' collaborations are a part of the task descriptions. Particularly, we consider that the task specification takes a form of a linear temporal logic formula, which may contain requirements and constraints on the other agent's behavior. A traditional automata-based approach to multi-agent strategy synthesis from such specifications builds on centralized planning for the whole team and thus suffers from extreme computational demands. In this work, we aim at reducing the computational complexity by decomposing the strategy synthesis problem into short horizon planning problems that are solved iteratively, upon the run of the agents. We discuss the correctness of the solution and find assumptions, under which the proposed iterative algorithm leads to provable eventual satisfaction of the desired specifications.
\end{abstract}

\section{Introduction}

In recent years, a considerable amount of attention has been devoted to automatic synthesis of robot controllers to execute complex, high-level mission, such as ``periodically survey regions $A$, $B$, $C$, in this order, while avoiding region~$D$'', specified as temporal logic formulas. 
Many of the suggested solutions to this problem and its variants rely on a three-step hierarchical procedure\cite{hadas09TL, marius-tac2008, nok-hscc2010, kavraki-ram}: First, the dynamics of the robotic system is abstracted into a finite, discrete transition system using e.g., sampling or cell decomposition methods based on triangulations or rectangular partitions. Second, leveraging ideas from formal verification methods, a discrete plan that meets the mission is synthesized. Third, the discrete plan is translated into a controller for the original system.

In this work, we focus on a multi-agent version of the above problem. Namely, we consider a team of robots, that are assigned a temporal mission each. 
As the robots may not be able to accomplish the mission without the help of the others, the individual mission specifications may contain requirements or constraints on the other team members' behavior. For instance, consider a warehouse solution with two mobile robots that periodically load and unload goods in certain locations of the warehouse. A part of the first robot's mission is to load an object in region $A$, however it is not able to load it by itself. Therefore at that point, the part of the mission is also a task for the second robot, to help loading.

The goal of this paper is to  synthesize a plan for each agent, such that each agent's mission specification is met. We follow the hierarchical approach to robot controller synthesis as outlined above and we narrow our attention to the second step of the approach, i.e., to generating discrete plans. The application of the algorithm that we propose is, however, not restricted to discrete systems: For the first step of the hierarchical approach, methods for discrete modeling of robotic systems can be used (\eg\cite{hadas09TL, marius-tac2008, nok-hscc2010, lavalle} and the references therein); for the third step, low-level controllers exist that can drive a robot from any position within a region to a goal region (\eg\cite{Belta-TAC06}).
As a mission specification language, we use Linear Temporal Logic (LTL), for its resemblance to natural language \cite{hadas-icra2012}, and expressive~power. 

Multi-agent planning from temporal logic specification has been explored in several recent works. 
\new{Planning from computational tree logic was considered in~\cite{quo-icra2004}, whereas in~\cite{loizou-cdc2005,marius-cdc2011}, the authors focus on planning behavior of a team of robots from a single, global LTL specification. 
Fragments of LTL have been considered for vehicle routing problems for unmanned aerial vehicles 
in~\cite{sertac-ijnc2010}, and for search and rescue missions in~\cite{lygeros-ecc2013}. A decentralized control of a robotic team from local LTL specification with communication constraints is proposed in~\cite{dimos-cdc12}. However, the specifications there are truly local and the agents do not impose any requirements on the other agents' behavior. Thus, the focus of the paper is significantly different to ours.
 }
As opposed to our approach, in~\cite{yushan-tr2012,alphan-ijrr2013}, a top-down approach to LTL planning is considered; the team is given a global specification and an effort is made to decompose the formula into independent local specifications that can be treated separately for each agent.

In~\cite{meng-cdc2013}, bottom-up planning from LTL specifications is considered, and a partially decentralized solution is proposed that takes into account only clusters of dependent agents instead of the whole group. A huge challenge of the previous approach is its extreme computational complexity. 
To cope with this issue, in this paper, we propose a receding horizon approach to multi-agent planning. The idea is to translate infinite horizon planning into an infinite sequence of finite horizon planning problems similarly as in \cite{nok-hscc2010}, where the authors leverage the same idea to cope with uncertain elements in an environment in single-robot motion planning. 
\new{To guarantee the satisfaction of the formula, we use an attraction-type function that guides the individual agents towards a progress within a finite planning horizon; similar ideas were used in~\cite{dennis-rh,maja} for a single-agent LTL planning to achieve a locally optimal behavior.}
The contribution of this paper can be summarized as the introduction of an efficient, limited horizon planning technique in the context of bottom-up control strategy synthesis for multi-agent systems from local LTL specifications. To our best knowledge, such an approach has not been taken to address the distributed multi-agent planning problem and its extreme computational demands before.

\new{The rest of the paper is structured as follows. In Sec.~\ref{sec:prelims}, we fix necessary preliminaries. Sec.~\ref{sec:pf} introduces the problem statement and summarizes our approach. In Sec.~\ref{sec:solution}, the details of the solutions are provided. We present an illustrative example and simulation results in Sec.~\ref{sec:simulations}, and we conclude and outline several directions for future research in Sec.~\ref{sec:summary}. 
}

\section{Preliminaries}
\label{sec:prelims}

Let 
$2^\Set$,
and $\Set^\omega$
denote 
the set of all subsets of a set $\Set$, 
and the set of all infinite sequences of elements of $\Set$, respectively.

\subsection{System Model and Specification}
\begin{definition}[Transition System]
A \emph{labeled deterministic transition system (TS)} is a tuple $\T=(S,s_{init},R,\AP,L)$, where
\begin{itemize}
\item 
$S$ is a finite set of states;
\item 
$s_{init} \in S$ is the initial state;
\item 
$R \subseteq S \times S$ is a deterministic transition relation;
\item 
$\AP$ is a set of services;
\item 
$L: S \rightarrow 2^\AP$ is a labeling function.
\end{itemize}
\label{def:TS}
\end{definition}

The labeling function assigns to each state $s$ a subset of \emph{services} that are available in that state. In other words, there is an option to provide or not to provide a service $\pi \in L(s)$ in the state $s$. In contrast, $\pi $ cannot be provided in $s$, where $\pi \not \in L(s)$.
The transition system evolves as follows: from a current state, either a subset of available services is provided, or the system changes its state by executing a transition while providing a so-called \emph{silent service} $\varepsilon$. Note, that we distinguish between an empty set of services $\emptyset$ and a silent service $\varepsilon$.
Formally, a \emph{trace} of $\T$ is an infinite 
alternating sequence of states and subsets of services $\tau = s_1\varpi_1s_2\varpi_2\ldots$, such that
$s_1 = s_{init}$, and for all $i \geq 1$ either
(i) $s_i = s_{i+1}$, and $\varpi_i \subseteq L(s_i)$,  or
(ii) $(s_i,s_{i+1}) \in R$, and $\varpi_i = \varepsilon$.

A trace $\tau=s_1\varpi_1s_2\varpi_2\ldots $ is associated with a sequence  $w_\varepsilon(\tau) = \varpi_1\varpi_2\ldots  \in (2^\Pi \cup \{\varepsilon\})^\omega$, and the \emph{word produced by} $\tau$ defined as the subsequence of the non-silent elements of $w_\varepsilon(\tau)$. Formally, a word produced by $\tau=s_1\varpi_1s_2\varpi_2\ldots$ is  $w(\tau) = \varpi_{i_1}\varpi_{i_2}\ldots \in (2^{\Pi })^\omega$, such that $\varpi_1, \ldots, \varpi_{i_1-1} = \varepsilon$, $\varpi_{i_j+1}, \ldots, \varpi_{i_{j+1}-1} = \varepsilon$  and $\varpi_{i_j} \neq \varepsilon$, for all $j \geq 1$. \new{The sequence of indexes $\mathbb{T}(\tau) = \mathbb{T}(w_\epsilon(\tau)) = i_1i_2\ldots$ is the sequence of time instances, when non-silent services are provided, called a \emph{service time sequence}. Note that the word $w(\tau)$ and the service time sequence $\mathbb{T}(\tau)$ might be finite as well as infinite.} However, 
as in this work we are interested in infinite, recurrent behaviors, we will consider as \emph{valid} traces only those producing infinite words.

{\begin{definition}
An LTL formula $\phi$ over the set of services $\AP$ is defined
  inductively as follows:
   \begin{enumerate}
  \setlength{\itemsep}{1pt}
  \setlength{\parskip}{0pt}
  \setlength{\parsep}{0pt}
  \item every service $\prop \in \AP$ is a formula, and
  \item if $\phi_1$ and $\phi_2$ are formulas, then $\phi_1 \lor
    \phi_2$, $\lnot \phi_1$, $\Next\, \phi_1$, $\phi_1\,\Until\,\phi_2$, $\Event \, \phi_1$, and $\Always \, \phi_1$
    are each formulas,
  \end{enumerate}
 where $\neg$ (negation) and $\vee$
  (disjunction) are standard Boolean connectives, and $\Next$ (next), $\Until$ (until), $\Event$ (eventually), and  $\Always$ (always) are temporal operators.
  \end{definition}}

The semantics of LTL is defined over infinite words over~$2^\AP$, such as those produced by traces of the TS from Def.~\ref{def:TS} (see, e.g.,~\cite{principles} for details). Intuitively, $\pi$ is satisfied on a word $w = w(1)w(2)w(3)\ldots$ if it holds at $w(1)$. Formula $\Next \, \phi$ holds true if $\phi$ is satisfied on the word suffix $w(2)w(3)\ldots$
, whereas $\phi_1 \, \Until\, \phi_2$ states that $\phi_1$ has to be true until $\phi_2$ becomes true. Finally, $\Event \, \phi$ and $\Always \, \phi$ are true if $\phi$ holds on $w$ eventually, and always, respectively.

The language of all words that are accepted by an LTL formula $\phi$ is denoted by $\Lang(\phi)$. A trace $\tau$ of
$\T$ satisfies LTL formula $\phi$, denoted by $\tau \models \phi$ iff the word $w(\tau)$ satisfies $\phi$, denoted $w(\tau) \models \phi$. 

\begin{remark}
{Traditionally, LTL is defined over the set of atomic propositions (APs) instead of services (see, e.g.~\cite{principles}). In transition systems, the APs represent inherent properties of system states. The labeling function $L$ then partitions APs into those that are true and false in each state.
 The  LTL formulas are interpreted over runs, i.e., sequences of states of transition systems. Run $s_1s_2\ldots$ satisfies $\phi$ if and only if the $L(s_1)L(s_2)\ldots \models \phi$.
}

In this work, we consider an alternative definition of LTL semantics to describe the desired tasks. Particlularly, we perceive atomic propositions as offered services rather than undetachable inherent properties of the system states. For instance, given that a state is determined by the physical location of an agent, we consider atomic propositions of form ``in this location, data can be gathered'', or ``there is a recharger in this location" rather than ``this location is dangerous''. In other words, the agent is given the option to decide whether an atomic proposition $\prop \in L(s)$ is in state $s$ satisfied or not. In contrast, $\prop \in \AP$ is never satisfied in state $s$, such that $\prop \not \in L(s)$. The LTL specifications are then interpreted over sequences of executed services along traces instead of the words produced by the traces.
\end{remark}

\subsection{Strategy Synthesis}
\label{sec:prelims:synthesis}
Given a transition system $\T$ with the set of atomic propositions $\AP$ and an automaton $\A$ over $2^\AP$, we say that a trace $\tau$ of $\T$ \emph{satisfies} $\A$, denoted by $\tau\models \A$ if and only if the word produced by $\tau$ belongs to the language of $\A$, \ie if $w(\tau) \in L(\A)$.

\begin{definition}[B\"uchi Automaton]
A B\"uchi automaton (BA) is a tuple $\B =  (Q,q_{init},\Sigma,\delta,F)$, where
\begin{itemize}
\item 
$Q$ is a finite set of states; 
\item 
$q_{init}\in Q$ is the initial state; 
\item 
$\Sigma$ is an input alphabet; 
\item 
$\delta \subseteq Q \times \Sigma \times Q$ is a non-deterministic transition relation; 
\item 
$F$ is the acceptance condition.
\end{itemize}
\end{definition}

The semantics of B\"uchi automata are defined over infinite words over $\Sigma$, such as those generated by a transition system from Def.~\ref{def:TS} if $\Sigma = 2^\Pi$. A \emph{run} of the BA $\B$ \emph{over} an input word $w=w(1)w(2)\ldots$  is a sequence
$\rho=q_1q_2\ldots$, such that $q_1  = q_{init}$, and
$(q_i,w(i),q_{i+1}) \in \delta$, for all $i\geq 1$. Word $w$ is accepted if there exists an accepting run $\rho$ over $w$ that intersects $F$ infinitely many times. 
$\Lang(\B)$ is the \emph{language} of all accepted words.
Any LTL formula $\phi$ over $\Pi$ can be translated into a BA $\B$, such that $\Lang(\B) = \Lang(\phi)$~\cite{principles} using an off-the-shelf software tool, such as~\cite{ltl2ba}.

Given a BA $\B$, we define the set of states $\hat \delta^k(q)$ that are reachable from a state $q \in Q$ in exactly $k$ steps inductively as
(i) $\hat \delta^0(q) = \{q\}$, and
(ii) $\hat \delta^{k+1}(q) = \bigcup_{q' \in \hat \delta^{k}(q)} \{q'' \mid \exists \, \sigma \in \Sigma. \, (q',\sigma,q'') \in \delta\}$, for all $k \geq 0$.

\begin{definition}[Product Automaton]
\label{def:product}
A product of a transition system $\T=(S,s_{init},R,\AP,L)$ and a B\"uchi automaton $\B =  (Q,q_{init},2^\Pi,\delta,F)$ is an automaton $\P =\T \otimes \B=  (Q_\P,q_{init,\P}, \Sigma_\P,  \delta_\P,F_\P)$, where
\begin{itemize}
\item 
$Q_\P = S \times Q$;
\item 
$q_{init,\P} = (s_{init},q_{init})$;
\item 
$\Sigma_\P = 2^{\Pi} \cup \{\varepsilon\}$;
\item 
$((s,q),\sigma, (s',q')) \in \delta_\P$ iff either
\begin{itemize}
\item[$\circ$] 
$(s,s') \in R, \sigma = \varepsilon$, $q=q'$, or
\item[$\circ$] 
$s=s', \sigma \subseteq L(s)$, and $(q,\sigma,q') \in \delta$;
\end{itemize}
\item 
$F_\P = S \times F$.

\end{itemize}
\label{def:product}
\end{definition}

A run of the product automaton over a word $w= \sigma_1\sigma_2\ldots \in 2^{\Pi }$ is a sequence of states $\rho = p_1p_2\ldots$, where $p_1=q_{\init,\P}$, with the property that and there exists a word $w_\varepsilon = w_\varepsilon(1)w_\varepsilon(2) \ldots =  \varepsilon \ldots \varepsilon \, \sigma_1 \varepsilon \ldots \varepsilon \, \sigma_2 \varepsilon \ldots \in \Sigma_\P^\omega$, such that $(p_i,w_\varepsilon(i),p_{i+1}) \in \delta_\P$, for all $i \geq 1$. Such a run is accepting if it intersects $F_\P$ infinitely many times. 

An accepting run $\rho = (s_1,q_1)(s_2,q_2)\ldots$ over a word $w = \sigma_1\sigma_2\ldots \in (2^\Pi)^\omega$ of the product automaton projects onto a valid trace $\tau=s_1 \varpi_1 s_2\varpi_2\ldots$ of $\T$, which produces the word $w$. At the same time, $w \in \Lang(\B)$. Dually, there exists an accepting run of $\P$ over each word $w \in \Lang(\B)$ that is produced by a valid trace $\tau$ of $\T$.

\medskip

An automaton $(Q,q_\init,\Sigma,\delta,F)$, can be viewed as a graph $(V,E)$ with the set of vertices $V= Q$ and the set of edges $E$ given by the transition function $\delta$ in the expected way. Thus, the standard notation from graph theory can be applied:
A \emph{path} in an automaton is a finite sequence of states and transition labels $q_i\xrightarrow{\sigma_i}q_{i+1}\ldots q_{l-1}\xrightarrow{\sigma_{l-1}}q_l$, such that $(q_j,\sigma_j, q_{j+1})\in \delta$, for all $i \leq j< l$. A path is \emph{simple} if $q_j = q_{j'} \Rightarrow j=j'$, for all $i\leq j,j' \leq l$. A path $q_i\xrightarrow{\sigma_i}\ldots q_l\xrightarrow{\sigma_l}q_{l+1}$, where $q_i \ldots q_l$ is a simple path and $q_{l+1}=q_i$, is called a \emph{cycle}.

Let $\suc(q)= \{q' \mid \exists \sigma.\, (q,\sigma,q')\in \delta\}$ denote the set of successors of $q$.
Furthermore, let $dist(q,q')$ denote the length of the shortest simple path that begins in $q$ and ends in $q'$, \ie the minimal number of states in a sequence representing a path $q\ldots q'$. If no such path exists, then $dist(q,q')=\infty$. If $q = q'$, then $dist(q,q')=0$. A \emph{shortest} path from $q$ to $q'$ is a path minimizing $dist(q,q')$, and can be computed using, e.g., Dijkstra algorithm~(see, e.g., \cite{cormen} for details).

\new{Given a product automaton $\P = \T \otimes \B$, a valid trace of $\T$ satisfying the specification $\B$ can be generated by finding a simple path from $q_{\init,\P}$ (a trace prefix) to an accepting state $q_f$ and a cycle $q_f \xrightarrow{\sigma_i} \ldots \xrightarrow{\sigma_l} q_f$, which contains at least one non-silent $\sigma_j \in 2^\Pi$, for some $j \in \{i,\ldots,l\}$ (a periodically repeated trace suffix). Such a simple path and cycle can be found using efficient graph algorithms.}

\section{Problem Formulation and Approach}
\label{sec:pf}

In this section, we formally state our problem of multi-agent planning from individual LTL specifications. 
We outline the straightforwards solution based on the control strategy synthesis method presented in Sec.~\ref{sec:prelims:synthesis}, and we discuss the drawbacks of this solution. Finally, to cope with these drawbacks, we suggest an alternative appraoch that is futher elaborated in details in Sec.~\ref{sec:solution}.

\subsection{Problem Statement}

Let us consider $N$ agents, (e.g., robots in a partitioned environment). Each agent is modeled as a finite transition system $\T_i=(S_i,s_{init,i},R_i,\AP_i,L_i)$, for all $i \in \{1,\ldots, N\}$. States of the transition system correspond to states of the agents (\eg the robot's physical location in the regions of the environment) and the transitions between them correspond to the agent's capabilities to change the state (\eg the ability of the robots to move between two regions of the environment). We assume that $(s,s) \in R_i$, for all $s\in S_i$, i.e., that any agent $i$ can stay in its current state, and we assume that each state $s \in S_i$ is reachable from all states $s'\in S_i$, i.e., that any agent can return to a state where it already was in the past. 
\new{We consider that the agents' transitions are synchronized in time; they are triggered at the same time instant and whenever a transition of one agent is triggered, then a transition of every other agent is triggered as well.}
Without loss of generality, we assume that $\AP_i \cap \AP_j = \emptyset$, for all $i\neq j \in \{1,\ldots,N\}$\new{, and that the set of silent services is $\mathcal E = \{\varepsilon_i \mid i \in \{1,\ldots,N\}\}$.}

Each agent is given an LTL task $\phi_i$ over 
$\APs_i = \bigcup_{j \in d(i)} \AP_j$, for some $\{i\} \subseteq d(i) \subseteq \{1,\ldots,N\}$. Informally, the satisfaction of an agent's task depends on, and only on the behavior of the subset of agents $d(i)$, including the agent itself.
Formula $\phi_i$ is interpreted over the traces $\tau_j = s_{j,1}\varpi_{j,1}s_{j,2}\varpi_{j,2}\ldots$ of transition systems $\T_j$, where $j \in d(i)$. 
\new{More precisely, the agent $i$ decides the satisfaction of the formula $\phi_i$  based on the word $w(\tau_i)$ it produces and on the 
services of agents $\T_j, j\in d(i)$ provided at the time instances $\mathbb T(\tau_i)$. In other words, the agent $\T_i$  observes and takes into consideration the other agents' services only at the time instances, when $\T_i$ provides a service (even an empty one) itself. Formally, let 
$\mathfrak w_\epsilon = \varpi_1 \varpi_2 \ldots \in (2^{\APs_i} \cup \mathcal E)^\omega$, where $\varpi_k = \bigcup_{j \in d(i)} \varpi_{j,k}$ denote the sequence of (silent and non-silent) services associated with the set of traces $\mathfrak{T}_i=\{\tau_j \mid j \in d(i)\}$, and let $\mathbb T(\tau_i) = \mathbb T(w_\epsilon(\tau_i)) = k_1k_2\ldots$. The \emph{word produced} by $\mathfrak{T_i}$  is then a sequence 
\begin{align}
\label{eq:teamword}
& w(\mathfrak{T}_i) = w(\mathfrak w_\epsilon) = \omega_{k_1}\omega_{k_2}\ldots, \\ \nonumber & \text{such that }  \omega_{k_m} =  \varpi_{k_m} \cap 2^{\APs_i},\text{ for all } m\geq 1.
\end{align}
}

The set of traces $\mathfrak{T}_i$ is called valid if the word $w(\mathfrak{T}_i)$ is infinite, i.e. if $\tau_i$ is valid. The formula $\phi_i$ is satisfied on a valid set of traces $\mathfrak{T}_i$, if and only if $w(\mathfrak{T}_i) \models \phi_i$.

\begin{example}
Consider transition systems $\T_1, \T_2$, with $\AP_1 = \{a\}$, and $\AP_2 = \{b\}$, and their tasks $\phi_1 = a \, \wedge \, \Next \,  (a \, \wedge \, b)$, $\phi_2 = b \, \wedge \, \Next \, (b \, \wedge \, a)$. Note that both $1 \in d(2)$, and $2 \in d(1)$. For traces $\tau_1$, $w_\epsilon(\tau_1) = \{a\} \epsilon \epsilon \{a\} \{\} \epsilon \ldots$, and $\tau_2$,  $w_\epsilon(\tau_2) = \epsilon \epsilon \epsilon \{b\} \{b\} \epsilon \ldots$, formula $\phi_1$ is satisfied, as the word produced by $\mathfrak T_1$ is $w(\mathfrak T_1) = \{a\}\{a,b\}\{b\}\ldots$. In contrast, $\phi_2$ is not satisfied, because $w(\mathfrak T_2) = \{a,b\}\{b\}\ldots$. Both formulas are satisfied if $w_\epsilon(\tau_1)$ changes to $\{a\} \epsilon \epsilon \epsilon \{a\} \epsilon \ldots$
\end{example}

\begin{problem}
\label{prob:main}
\emph{Given} $N$ agents represented as transition systems $\T_i=(S_i,s_{init,i},R_i,\AP_i,L_i)$, and LTL formulas $\phi_i$ over
$\APs_i = \bigcup_{j \in d(i)} \AP_j$, for all $i \in \{1,\ldots, N\}$, \emph{find} a trace $\tau_i$ of each $\T_i$, such that $\mathfrak{T}_i = \{\tau_j \mid j \in d(i)\}$ is valid and satisfies the specification $\phi_i$, for all $i \in \{1,\ldots, N\}$.
\end{problem}
As each of the LTL formulas $\phi_i$, $i \in \{1,\ldots, N\}$ over $\APs_i$ can be translated into a language equivalent B\"uchi automaton, we can pose the problem equivalently as: 

\begin{problem}
\label{prob:main2}
\emph{Given} $N$ agents represented as transition systems $\T_i=(S_i,s_{init,i},R_i,\AP_i,L_i)$, and 
B\"uchi automata $\B_i = (Q_i, q_{\init, i}, \delta_i, \mathbf{\Sigma}_i=2^{\APs_i}, F)$, 
for all $i \in \{1,\ldots, N\}$, \emph{find} a trace $\tau_i$ of each $\T_i$, such that $\mathfrak{T}_i = \{\tau_j \mid j \in d(i)\}$ is valid and produces a word $w(\mathfrak{T}_i) \in \Lang(\B_i)$.
\end{problem}

\subsection{Straightforward Centralized Solution}
\label{sec:pf:cs}

An immediate solution to the Prob.~\ref{prob:main2} can be obtained by a slight modification to the standard control strategy synthesis procedure for transition systems from LTL specification (see Sec.~\ref{sec:prelims:synthesis}). 
Roughly, the procedure solving Prob.~\ref{prob:main2} include (1) partitioning the  set of agents into dependency classes similarly as in \cite{meng-cdc2013}, by iterative application of the rule that if $j \in d(i)$, then $\T_j$ belongs to the same dependency class as $\T_i$; (2) for each dependency class $D =\{ \T_{d_1},\ldots,\T_{d_m} \}$, constructing a transition system $\T_D$ with the set of states $S = S_{d_1}\times \ldots \times S_{d_m}$ that represents the synchronized behavior of agents within the class; (3) building a B\"uchi automaton $\B_D$, which accepts all the sequences $\mathfrak w_\varepsilon = \varpi_1\varpi_2 \ldots \in (2^{\APs_{D}} \cup \mathcal E)^\omega$, such that the produced word $w(\mathfrak w_\epsilon) \in (2^{\APs_D})^\omega$ (see Eq.~\ref{eq:teamword}) satisfies $\phi_i$, for all $T_i \in D$; (4)~constructing a product automaton $\P_D$ of $\T_D$ and $\B_D$; and (5) using  graph algorithms to find an accepting run of $\P_D$ that projects onto valid traces of $\T_i$, and accepting runs of $\B_i$, for all $T_i \in D$.

The outlined procedure is correct and complete; a solution is found if one exists and it is indeed a solution to Prob.~\ref{prob:main}. However, it suffers from a rapid growth of the product automaton state space with the increasing number of agents, leading to extreme computational demands that make the approach infeasible in practice. Particularly, if the size of $D$ is $N$, the product automaton $\T_D \otimes \B_D$ is $\mathcal{O}(\prod_{1\leq i \leq N} |\T_i|)$, which is approx. $|T|^N$.

\subsection{Our Approach}

In this work, we aim on reducing the high computational complexity of the straightforward solution. Our approach is to avoid the execution of an offline, centralized control strategy generation procedure and to decompose the strategy synthesis problem into short horizon planning problems that are solved online, upon the execution of the system, similarly as in model predictive control. As a starting point, we consider the problem definition from Prob.~\ref{prob:main2}.

In the sequel, we present an iterative method to select a temporary goal state for each agent within a short horizon, and compute and execute a finite trace fragment leading to this goal state. We show, that under certain assumptions, the repetitive implementation of the outlined algorithm leads to provable satisfaction of the desired specifications.
The solution leverages ideas from LTL control strategy synthesis and also the construction of intersection B\"uchi automata~\cite{principles}.
\section{Problem Solution}
\label{sec:solution}

In this section, we provide details of the proposed solution to Prob.~\ref{prob:main2}. First, we introduce the procedures that are executed in each iteration of the algorithm, followed by the summary of the overall method. Along the procedures presentations, two assumptions are imposed to ensure the correctness of the algorithm and we discuss how they can be relaxed towards the end of this section.

Besides the set of transition systems $\T_1, \ldots, \T_N$, and the specification automata $\B_1,\ldots, \B_N$, the inputs to each iteration of the algorithm are:
\begin{itemize}
\item current states of $\T_1,\ldots,\T_N$, denoted by $\currs_1, \ldots, \currs_N$, initially equal to $s_{\init,1},\ldots, s_{\init, N}$, respectively;
\item current states of $\B_1,\ldots,\B_N$, denoted by $\currq_1, \ldots, \currq_N$, initially equal to $q_{\init,1},\ldots, q_{\init, N}$, respectively;
\item linear ordering $\prec$ over $\{1,\ldots, N\}$, initially arbitrary;
\item a fixed horizon $\h \in \Nat$, which, loosely speaking, determines the depth of planning in the B\"uchi automata;
\item a fixed horizon $\H \in \Nat$ which, loosely speaking, determines the depth of planning in the transition systems.
\end{itemize}

\subsection{Intersection B\"uchi Automata}
\label{sec:intersection}

In each iteration of the algorithm, we construct local automata that represent the intersection of relevant B\"uchi automata up to a pre-defined horizon $\h$. We label their states with values that, simply put, indicate the progress towards the satisfaction of the desired properties. 
Later on, these values are used to set local goals in the  short horizon planning. 

We partition the set of B\"uchi automata $\Spec = \{\B_1,\ldots, \B_N\}$ into the smallest possible subsets $\Phi_1,\ldots,\Phi_M$, such that any transition of any $\B_i \in \Phi_\ell$ up to horizon $\h$ from the current state does not impose restrictions on the behavior of any agent $\T_j$ with the property that $\B_j \not \in \Phi_\ell$. {This partition corresponds to the current \emph{necessary and sufficient dependency} between agents up to the horizon $\h$, and can dynamically change over the time.}

\begin{definition}[Participating Services] Formally, we call a set of services $\AP_j$, $j \in d(i)$  \emph{participating} in $q\in Q_i$ if 
\begin{itemize}
\item[(i)] $j=i$, or 
\item[(ii)] there exist  $q'\in Q_i$,  $\sigma \in \Alpha_i$, and $\varsigma \subseteq \AP_j $ such that $(q,\sigma, q') \in \delta$, and $(q, (\sigma \setminus \AP_j) \cup \varsigma, q') \not \in \delta$.
\end{itemize} 
\end{definition}
Intuitively, a set of services $\AP_j$ is participating in $q$, if some transition leading from $q$ imposes restrictions on the services provided by agent $j$. 
\begin{definition}[Alphabet up to Horizon $h$] For a state $q \in Q_i$, we define the alphabet $\Alpha_i^{\h}$ of $\B_i$ up to the horizon $\h$ as $\Alpha_i^{\h}(q) = 2^{\APs_i^h(q)}$, where 
$$\APs_i^h(q) = \bigcup_{\substack{q' \in \hat \delta_i^k(q)\\0\leq k \leq \h}} \{\AP_j \mid \AP_j \text{ is a participating service in } q'\}.$$
\end{definition}
\begin{definition}[Dependency Equivalence and Partition] 
\label{def:dep}
Given that $\currq_1, \ldots, \currq_N$ are the respective current states of B\"uchi automata $\B_1,\ldots, \B_N$, the partition of the set of B\"uchi automata $\Spec$ is induced by the dependency equivalence $\sim^h$ defined on $\Spec$ as follows:
\begin{itemize}
\item $\B_i \sim^h \B_i$
\item if there exists $\B_k$, such that $\B_i \sim^h \B_k$, and $\AP_j  \subseteq \APs_k^h(\currq_k)$ or $\AP_k  \subseteq \APs_j^h(\currq_j)$, then also $\B_i \sim^h \B_j$.
\end{itemize}
The desired partition is then $\{ \Phi_1,\ldots, \Phi_M \}$, with the property that $(\B_i \sim^h \B_j) \iff (\B_i \in \Phi_\ell \iff \B_j \in \Phi_\ell)$. 
We associate each subset of B\"uchi automata $\Phi_\ell$ with the set of indexes $I_\ell$, such that $\B_i \in \Phi_\ell \iff i \in I_\ell$.
\end{definition}

Note, that planning within the horizon $\h$ can now be done separately for each $\Phi_\ell$. 
Thus, from now on, in the remainder of this section and Sec.~\ref{sec:product} and~\ref{sec:plan}, let us concentrate on planning for a \emph{dependency class} of agents and specifications given by $I_\ell = \{1_\ell, \ldots, n_\ell\}$, for a fixed $\ell$.

\smallskip

We are now ready to define the B\"uchi automata intersection up to the horizon $\h$, for  $\Phi_\ell = \{\B_{1_\ell},\ldots ,\B_{n_\ell}\}$. Let ${i_\ell} \prec {j_\ell}$, for all $1\leq i<j \leq n$. In other words, we assume, without loss of generality, that the automata in $\Phi_\ell$ are ordered according to $\prec$.
\begin{definition}[Intersection Automaton]
\label{def:BA}
$ $ \\ The intersection automaton of $\B_{1_\ell}, \ldots, \B_{n_\ell}$ up to horizon $\h$ is $\A^h = (Q_\A, q_{\init, \A}, \Alpha_\A, \delta_\A, F_\A)$, where
\begin{itemize}
\item $Q_\A \subset Q_{1_\ell} \times \ldots \times Q_{n_\ell} \times \Nat$ is a finite set of states, generated as described below;
\item $q_{\init, \A} = (\currq_{1_\ell}, \ldots, \currq_{n_\ell}, 1)$;
\item \new{$\Alpha_\A = \{ \sigma \in 2^{\mathbf{\Pi}_\A} \mid \forall i_\ell \in I_\ell.  \ \sigma \cap 2^{\mathbf{\Pi}_{i_\ell}} \neq \emptyset \Rightarrow \epsilon_{i_\ell} \not \in \sigma\}$, where $\mathbf{\Pi}_\A = \bigcup_{i \in I_\ell} \AP_i \cup \bigcup_{i \in I_\ell} \{\epsilon_i\}$;}
\item  Let $Q_\A^0 = \{ q_{\init,\A} \}$. \\
For all $1 \leq j \leq h$, we define $(q_{1_\ell}', \ldots, q_{n_\ell}', k') \in Q_\A^{j}$ and $\big((q_{1_\ell}, \ldots, q_{n_\ell}, k), \sigma, (q_{1_\ell}', \ldots, q_{n_\ell}', k')\big) \in \delta_\A^j$ iff
\begin{itemize}
\item[i)] $(q_{1_\ell}, \ldots, q_{n_\ell}, k) \in Q_\A^{j-1}$, 
\item[ii)] \new{for all $i_\ell \in I_\ell$, either
\begin{itemize}
\item[$\circ$]$(q_{i_\ell}, \sigma \cap \APs_{i_\ell} ,q_{i_\ell}') \in \delta_{i_\ell}$, or
\item[$\circ$] $q_{i_\ell} = q_{i_\ell}'$, and $\epsilon_{i_\ell} \in \sigma$
\end{itemize}}
\item [iii)] \[ k' = \left\{ 
  \begin{array}{l l}
   k+1 & \quad \text{if $q_{{(k \mod n)}_\ell}  \in F_{{(k \mod n)}_\ell}$},\\
   k & \quad \text{otherwise.}
  \end{array} \right.\]
\end{itemize}
Finally, $Q_\A = \bigcup_{0 \leq j \leq \h } Q_\A^j  \ \  \text{  and   } \  \ \delta_\A  = \bigcup_{1 \leq j \leq \h } \delta_\A^j; $
\item $F_\A = \{ (q_{1_\ell},\ldots,q_{n_\ell}, k) \in Q_\A \setminus \{q_{\init, \A}\} \mid $ \\ $~\ \ \ \ \ \ \ \ \ \ q_{(k \mod n)_\ell} \in F_{(k \mod n)_\ell} \}$.
\end{itemize}
\end{definition}
The intersection automaton is not a B\"uchi automaton as it does not exhibit infinite runs. However, it is an automaton that reads finite words and thus, it can be viewed as a graph. Through $k$, we remember which accepting states of which $\B_{i_\ell}$ have been visited on a run towards the respective state;  accepting states of all $\B_{1_\ell},\ldots, \B_{i_\ell}$ have been visited on each path from $q_{\init,\A}$ to the state with $k=i +1$. Thus, intuitively, the greater $k$ translates to the greater progress towards satisfaction of the individual specifications ordered according to $\prec$.

\begin{assumption} 
\label{assump:BA}
Assume that $F_\A$ is not empty. 
\end{assumption}

Intuitively, this assumption captures that at least a state which ensures a progress towards the satisfaction of the highest-order specification $\B_{1_\ell}$ is present in $\A$. This allows us to identify local goal states in $\T_{1_\ell}, \ldots, \T_{n_\ell}$ in the following subsection. Without this assumption, we would not be able to distinguish between ``profitable'' and ``profitless'' transitions of agents with respect to $\mathbf{\Phi}$.
We analyze conditions under which Assump.~\ref{assump:BA} can be violated and propose a solution to its relaxation in Sec.~\ref{sec:relax:BA}.

\begin{definition}[Progressive Function for $\A$]
\label{def:VB} The progressive function $V_{\A}: Q_\A \to \Nat_0 \times \mathbb{Z}_0^-$ is for a state $q=(q_{1_\ell}, \ldots, q_{n_\ell}, k)$ defined as follows:
\begin{align*}
V_{\A}(q) =  \big(k,-\min_{q_f \in F_\A} dist(q,q_f)\big).
\end{align*}
\end{definition}

The increasing value of $V_{\A}$ indicates a progress towards the satisfaction of the individual local specifications in $\Phi_\ell$, ordered according to $\prec$. No progress can be achieved from state $q$, such that $V_{\A}(q) = (k,-\infty)$ within the horizon $\h$, and hence, we remove these 
from $\A$. From Assump.~\ref{assump:BA}, we have that $V_{\A}(q_{\init,\A}) = (1, d)$, where $d\neq -\infty$.

\subsection{Product System}
\label{sec:product}
The intersection automaton and its progressive function allows us to define which services should be provided in order to make a progress towards satisfaction of the specification. The remaining step is to plan the transitions of the individual agents to reach states in which these services are offered. We do so through definition of a product system that captures the allowed behaviors (finite trace fragments) of agents from $I_\ell$ up to horizon $\H$. The states of the product system are evaluated based on the progressive function of $\A$, to indicate their progress towards satisfaction of the formula.

\begin{definition}[Product System]
\label{def:PA}
The product system up to the horizon $\H$ of the agent transition systems $\T_{i_\ell}, i_\ell \in I_\ell$, and the intersection B\"uchi automaton $\A^h$ from Def.~\ref{def:BA} is an automaton $\P^\H= (Q_\P, q_{\init,\P},  \Alpha_\P, \delta_\P)$, where

\begin{itemize}
\item $Q_\P \subset S_{1_\ell} \times \ldots \times S_{n_\ell} \times Q_\A$ is a finite set of states, generated as described below;
\item $q_{\init,\P} = (\currs_{1_\ell},\ldots,\currs_{n_\ell},q_{\init,\A})$;
\item \new{$\Alpha_\P =\Alpha_\A;$} 
\item Let $Q^0_\P = \{ q_{\init,\P} \}$. \\
For all $1 \leq j \leq H$, $(s_{1_\ell}', \ldots, s_{n_\ell}', q') \in Q_\P^{j}$ and
$\big((s_{1_\ell},\ldots,s_{n_\ell},q),\sigma,((s_{1_\ell}',\ldots s_{n_\ell}',q')\big)\in\delta_\P^j$
iff for all $i \in \{1,\ldots, n\}$, either
\new{$s_{i_\ell} = s_{i_\ell}', \sigma \cap \Pi_{i_\ell} \subseteq L(s_{i_\ell})$ and $(q,\sigma,q') \in \delta_\A$, or $(s_{i_\ell}, s_{i_\ell}') \in R_{i_\ell}$, $\epsilon_{i_\ell} \in \sigma$ and $(q,\sigma,q') \in \delta_\A$.}\\

Finally, $Q_\P = \bigcup_{0 \leq j \leq H} Q^j_\P \ \ \text{ and } \ \ \delta_\P  = \bigcup_{1 \leq j \leq H} \delta^j_\P. $
\end{itemize}
\end{definition}

The set of accepting states $F_\P$ is not significant for the further computations, hence we omit it from $\P^\H$. The tuple $\P^\H$ is an automaton and can be viewed as a graph (see Sec.~\ref{sec:prelims}). A path 
$p = q_1\xrightarrow{\sigma_1}q_2 \ldots q_{m-1}\xrightarrow{\sigma_{m-1}}q_m$ in $\P^\H$, where $q_1 = q_{\init,\P}$ can be projected onto a finite trace  prefix $\tau_{i_\ell}(p)$ of each $\T_{i_\ell}$, $i_\ell \in I_\ell$ in the expected way: the $j$-th state of $\tau_{i_\ell}(p)$ is $s_{i_\ell}$ if the $j$-th state of $p $ is $q_j=(s_{1_\ell}, \ldots, s_{n_\ell},q_\A)$, and the $j$-th set of services of $\tau_{i_\ell}(p)$ is $\sigma_j \cap (\Pi_{i_\ell} \cup \{\epsilon_{i_\ell}\})$, for all $j \in \{1,\ldots m\}$, and $j \in \{1,\ldots m-1\}$, respectively.
The path $p$ can be naturally projected onto a finite run prefix of the intersection automaton $\A^h$ and onto finite run prefixes of individual B\"uchi automata $\B_{i_\ell}$, too. Particularly, the $j$-th state of the run prefix $\rho_{\A}(p)$ of $\A^h$ is $q_\A = (q_{\A,1_\ell},\ldots, q_{\A,n_\ell}, k)$ if the $j$-th state of $p$ is $q_j = (s_{1_\ell},\ldots,s_{n_\ell},q_\A)$, for all $j \in \{1,\ldots, m\}$; the $j$-th state of the run prefix $\rho_{{i_\ell}}(p)$ of $\B_{i_\ell}$ is then the state $q_{\A,i_\ell}$.

\begin{definition}[Progressive Function and State] The progressive function $V_{\P}: Q_\P \to \Nat_0 \times \mathbb{Z}_0^-$ is inherited from the intersection automaton $\A^h$ (Def., \ref{def:VB}), \ie for all $(s_{1_\ell},\ldots, s_{n_\ell},q) \in Q_\P$,
$V_\P\big((s_{1_\ell},\ldots, s_{n_\ell},q)\big)= V_{\A}(q).$

A  state $q\in Q_\P$ is a \emph{progressive state} if $V_\P(q) > V_\P(q_{\init, \P})$. A \emph{maximally progressive state} is a progressive state $q$, such for all $q' \in Q_\P$, it holds $V_\P(q) \geq V_\P(q')$. 
\end{definition}

\subsection{Plan Synthesis}
\label{sec:plan}

Given $\P^\H$, we compute the local plan as the shortest path $p = q_1\xrightarrow{\sigma_1}q_2 \ldots q_{m}\xrightarrow{\sigma_{m}}q_{\mathit{max}}$ from $q_1=q_{\init,\P}$ to ~$q_{\mathit{max}}$, \new{such that $\sigma_k \cap \T_{1_\ell} \neq \emptyset$, for some $k \in \{1,\ldots, m\}$, and $q_\mathit{max}$ is a maximally progressive state reachable through such a path. The path can be computed using efficient graph algorithms, in linear time with respect to the size of $P^\H$.
We assume that such a path exists and show how to relax the assumption further  in Sec.~\ref{sec:relax:PA}. } 

\begin{assumption}
\label{assump:PA}
Assume that in $\P^\H$, there exists at least one progressive state $q_\mathit{p}$ reachable through a finite path $q_{\init,\P}\xrightarrow{\sigma_1}q_2 \ldots q_{m}\xrightarrow{\sigma_{m}}q_{\mathit{p}}$, such that $\sigma_k \cap \T_{1_\ell} \neq \emptyset$, for some $k \in \{1,\ldots, m\}$.
\end{assumption}

The projection of the found path onto individual agent transition systems gives finite trace prefixes $\tau_{i_\ell}(p) = \currs_{i_\ell} \varpi_{i_\ell,1} s_{i_\ell,2} \ldots s_{i_\ell,m}$, to be followed by each agent $i_\ell \in I_\ell$. \new{Furthermore, it is guaranteed that at least agent $\T_{1_\ell}$ will provide at least one non-silent service along its trace prefix.}

\subsection{Plan Execution}

Finally, in each iteration the individual trace prefixes $\tau_i(p) = \currs_{i} \varpi_{i,1} s_{i,2} \varpi_{i,2} \ldots s_{i,m_i}$ computed in the previous steps are executed as follows. Each agent $\T_{i}$, $i \in \{1,\ldots, N\}$ provides the services $\varpi_{i,1} \in L(\currs_i)$, and executes the transition to the state $s_{i,2}$.
At the same time, the current state of each B\"uchi automaton $\B_i$, $i \in \{1,\ldots, N\}$ is updated to the second state $q_{i,2}$ of the run prefix $\rho_i(p)=\currq_i q_{i,2}\ldots q_{i,m_i}$ obtained by the projection of $p$ onto $\B_i$. If $q_{i,2} \in F_i$, then the ordering $\prec$ is also updated, in such a way that $i$ becomes of the lowest order, \ie $j,j' \prec i$ for all $j,j' \in \{1,\ldots, N\} \setminus \{i\}$, while maintaining the mutual ordering of $j$ and $j'$. Loosely speaking, this change reflects that a progress towards the satisfaction of specification $\B_i$ has been made and in the following iteration, we focus on {making progress towards the satisfaction of the remaining specifications.}

\begin{algorithm}[!h]
\caption{Solution to Prob.~\ref{prob:main2}}
\label{alg:main}
\begin{algorithmic}[1]
\small
\INPUT Transition systems $\T_1,\ldots, \T_N$; B\"uchi automata $\B_1,\ldots, \B_N$; horizons $\h \in \Nat$, $H\in \Nat$.
\OUTPUT \emph{system execution} $(\tau_1,\ldots, \tau_N,\rho_1,\ldots,\rho_N)$, where \\ $\tau_1,\ldots,\tau_N$ are traces of  $\T_1,\ldots, \T_N$, and $\rho_1,\ldots, \rho_N$ are runs of $\B_1,\ldots,\B_N$, respectively
\STATE $\prec\,:=(1,\ldots,N)$; $\currs_i:=s_{\init,i}$ $\currq_i := q_{\init,i}, \forall i \in \{1,\ldots, N\}$
\WHILE {$\true$}
\STATE compute the partition $\{I_1,\ldots,I_M\}$ (Def.~\ref{def:dep})
\FORALL{ $\ell \in \{1,\ldots,M\}$}
\STATE construct $\A^h$ (Def.~\ref{def:BA})
\STATE construct $\P^H$ (Def.~\ref{def:PA})
\STATE find a shortest path $p$ to a max. progressive state in~$\P^\H$
\ENDFOR
\FORALL {$i \in \{1,\ldots, N\}$, suppose that \\
$\tau_i(p) = \currs_i\varpi_{i,1}s_{i,2}\ldots s_{i,m_i}, \rho_i(p) = \currq_iq_{i,2}\ldots,q_{i,m_i},$}
\STATE provide services $\varpi_{i,1} \in L(\currs_i)$
\STATE $\currs_i :=  s_{i,2}$; $\currq_i := q_{i,2}$
\IF {$\currq_i \in F_i$}
\STATE reorder $\prec$, s.t. $j \prec i$, for all $j \in \{1,\ldots,N\} \setminus \{i\}$
\ENDIF
\ENDFOR
\ENDWHILE 
\end{algorithmic}
\end{algorithm}

This step has finalized one iteration of the algorithm and at this point, the next iteration is to be performed, starting with building the intersection automaton in Sec.~\ref{sec:intersection}. The overall solution is summarized in Alg.~\ref{alg:main}.

\new{
\subsection{Correctness}

\begin{lemma}
A system execution $(\tau_1,\ldots, \tau_N, \rho_1,\ldots, \rho_N)$ computed by Alg.~\ref{alg:main} satisfies the following, for all $i \in \{1,\ldots, N\}$:
\begin{itemize}
\item[(i)] given that $\rho_i = q_{i,1}q_{i,2}\ldots$, $i \in \{1,\ldots, N\}$, and $\mathbb T({ \tau_i}) = k_1k_2\ldots$, the sequence $\varrho_i = q_{i,k_1} q_{i,k_2}\ldots$ is a run of $\B_i$, and furthermore $q_{i,1} = \ldots = q_{i,k_1-1}$, and $q_{i,k_j+1} = \ldots = q_{i,k_{j-1}-1}$, for all $j \geq 1$.
\item[(ii)] $\tau_i$ is a valid trace of $\T_i$.
\item[(iii)] $\rho_i$ contains infinitely many states $q_f \in F_i$.
\end{itemize}
\end{lemma}

\begin{proof}
Let $t$ be an arbitrary time instant, and let $\tau_1^t,\ldots,\tau_N^t, \varpi_1^t,\ldots,\varpi_N^t,\rho_1^t,\ldots,\rho_N^t$ denote the  current states and provided services of $\T_1,\ldots,\T_N$, and the current states of $\B_1,\ldots,\B_N$ at time $t$, respectively. Then, directly from the constructions of the intersection automaton and the product, we have the following: for all $i$, it holds that $(\tau_i^t,\tau_i^{t+1})\in R_i$. Furthermore, if $\varpi_i^t =\epsilon_i$, then $q_{i}^t = q_{i}^{t+1}$. On the other hand, if $\varpi_i^t \neq \epsilon$, then  $ q_{i}^{t+1} \in \delta(q_{i}^t, \bigcup_{j \in d(i)} \varpi_j^t)$.

Consider that $i$ is the most prioritized agent at time $t$, i.e., that $i \prec j$, for all $j \in \{1,\ldots, N\}$. Let $\tau_i'$ and $\rho_i'$ are the finite trace and run prefixes of $\T_i$, $\B_i$ computed by Alg.~\ref{alg:main} on lines~7--9 at time $t$ to a maximally progressive state $q_\mathit{max}$ of $\P^H$. Then, intuitively, at time $t+1$, this state is also present in $\P^H$. If a plan is changed to reach $q'_\mathit{max}$, then $q'_\mathit{max}$ is ``more progressive'' than $q_\mathit{max}$, and thus closer to reaching an accepting state of $\B_i$. 
Altogether, thanks to the Assump.~\ref{assump:BA} and Assump.~\ref{assump:PA}, we can state that a state $q_\mathit{max}$, which projects onto an accepting state $q_f$ of $\B_i$ is reached. At the same time, it is ensured that at least one non-silent service is provided by $\T_i$ on this path. Furthermore, lines 12--14 of Alg.~\ref{alg:main} ensure, that each $i \in \{1,\ldots,N\}$ will repeatedly become the most prioritized. Putting everything together, we can conclude that the lemma holds.
\end{proof}

\begin{corollary}
A system execution $(\tau_1,\ldots, \tau_N, \rho_1,\ldots, \rho_N)$ returned by Alg.~\ref{alg:main} provides a solution to Prob.~\ref{prob:main2}. 
\end{corollary}

}

\subsection{Relaxing the Assumptions}
\label{sec:relax}
\subsubsection{Relaxing Assump.~\ref{assump:BA}} 
\label{sec:relax:BA}
Intuitively, Assump.~\ref{assump:BA} may be violated from two different reasons: First, if the selected horizon $\h$ is too short, and although $F_\A = \emptyset$, there exists $\h' > h$, such that $F_\A \neq \emptyset$ in $\A^{h'}$. Second, if $F_\A = \emptyset$ even for $h \to \infty$, i.e., if a wrong step was executed in the past that lead to the infeasibility of the formula. We show, how to identify the reason of the assumption violation and propose an approach to its relaxation.

Consider that $\A^h$ is built according to Def.~\ref{def:BA} and that $F_\A= \emptyset$. In short, we systematically extend the horizon $h$ and update the automaton $\A^h$ until a set of states $F_\A$ becomes nonempty, or until the extension does not change the automaton $\A^h$ any more. 
 In the former case, the automaton $\A^h$ with the extended horizon satisfies Assump.~\ref{assump:BA} and thus is used in constructing $\P^H$, maintaining the remainder of the solution as described in Sec.~\ref{sec:product} and \ref{sec:plan}. In the latter case, the specification has become infeasible, indicating that a wrong step has been made in past.
Therefore, we backtrack along the executed solution to a point when another service could have been executed instead of the one that has been already done. Intuitivelly, we ``undo'' the service, we pretend that it has not been provided and mark this service as forbidden in the specification automata. The backtracking procedure is roughly summarized in Alg.~\ref{alg:backtrack}.

\begin{remark}
In order to perform the backtracking, the system execution prefixes have to be remembered. To reduce the memory requirements, note that cycles between two exact same system execution states can be removed from the system execution prefixes without any harm.

As there are only finitely many transitions possible in each system state of each transition system and each B\"uchi automaton, the backtracking procedure will ensure that eventually, the agents' trace prefixes will be found by Alg.~\ref{alg:main} without any further backtracking. Intuitively, this happens in the worst-case after the backtracking procedure rules out all the possible wrong transitions of the agents (line 5).
\end{remark}

\subsubsection{Relaxing Assump.~\ref{assump:PA}} 
\label{sec:relax:PA}
Once Assump.~\ref{assump:BA} holds, there is only one reason for violation of Assump.~\ref{assump:PA}, which is that the planning horizon $\H$ is not long enough. To cope with such a situation, we systematically extend the horizon $\H$ similarly as we extended $\h$ in the B\"uchi automaton. Eventually, a progressive state will be found.

\begin{algorithm}[!h]
\caption{Backtracking}
\label{alg:backtrack}
\begin{algorithmic}[1]
\footnotesize
\INPUT  Transition systems $\T_1,\ldots, \T_N$; B\"uchi automata $\B_1,\ldots, \B_N$; \emph{System execution prefix} $(\tau_1^{\mathfrak t},\ldots,\tau_N^{\mathfrak t},\rho_1^{\mathfrak t},\ldots,\rho_N^{\mathfrak t})$ up to the current time $\mathfrak t$, where $\tau_i^{\mathfrak t} = s_{i,1}\varpi_{i,1}\ldots \varpi_{i,{\mathfrak t}-1}s_{i,\mathfrak t}$, and $\rho_i^{\mathfrak t} = \rho_{i,1} \ldots \rho_{i,\mathfrak t}$, for all $i \in \{1,\ldots N\}$.
\OUTPUT Updates to B\"uchi automata $\B_1,\ldots, \B_N$
\STATE $k := \mathfrak t$
\WHILE {plan not found}
\STATE $k := k -1$
\STATE Check, if the execution of $\bigcup_{i \in \{1,\ldots,N\}} \varpi_{i,k}$ can lead to a different set of states of B\"uchi automata than to $q_{1,t},\ldots, q_{N,t}$. If so, apply the change and goto line 6.
\STATE Forbid the execution of $\bigcup_{i \in \{1,\ldots,N\}} \varpi_{i,k}$ in the states $q_{1,k},\ldots, q_{N,k}$ of each respective automaton $\B_1,\ldots,\B_N$
\STATE Execute one iteration of  Alg.~\ref{alg:main}, line 3--16, from $\currs_1 = s_{1,k},\ldots, \currs_N=s_{N,k}, \currq_1 = q_{1,k},\ldots, \currq_N=q_{N,k}$
\STATE If a plan was found in line 5, continue with execution of Alg.~\ref{alg:main}, otherwise goto line 2 of Backtracking.
\ENDWHILE
\end{algorithmic}
\end{algorithm}

\begin{remark}
Note, that Assump.~\ref{assump:BA} and~\ref{assump:PA} can be enforced by the selection large enough $h$ and $H$, respectively. Particularly, $h \geq \max_{i\in N} |Q_i|$, and $H \geq \max_{i\in N}|S_i|$ ensures the completeness of our approach. However, in such a case, the complexity of the proposed approach meets the complexity of the centralized solution {discussed in~Sec.~\ref{sec:pf:cs}}.
\end{remark}
\section{Example}
\label{sec:simulations}
To demonstrate our approach and its benefits, we present an illustrative example of three mobile robots operating in a common workspace depicted in Fig.~\ref{fig:example}.(A). The agents can transit in between the adjacent cells of the partitioned environment and they can each provide various services. Agent 1 can load ($l_H,l_A,l_B,l_C$), carry, and unload ($u_H,u_A,u_B,u_C$) a heavy object $H$ or a light object $A$, $B$, $C$. Agent 2 is capable of helping the agent 1 to load object $1$ ($h_H$), and to execute simple tasks in the purple regions ($t_1$ -- $t_5$). Agent~3 is capable of taking a snapshot of the rooms $R_1 - R_5$ when being present 
within the respective room ($s_1$ -- $s_5$). 

The robots are assigned complex tasks that require collaboration. Agent 1 would like agent 2 to help loading the heavy object. Then, it should carry the object to tje unloading point and unload it. After that, its task is to periodically load and unload all the light objects. The goal of agent 2 is to periodically execute the sequence of simple tasks $t_1,\ldots,t_5$, in this order. Furthermore, it requests agent 3 to witness the execution $t_5$, by taking a snapshot of room $R_4$ at the moment of the execution. Finally, the goal of agent 3 is to patrol rooms $R_2,R_4,R_5$.
The LTL formulas for the agents are:
$ \phi_1 = \Event (l_H \wedge h_H \wedge \Next \, u_H \wedge \bigwedge_{i \in \{A,B,C\}} \Always \Event \, (l_i \wedge \Next u_i)),$\\
$ \phi_2 = \Always \Event \ (t_1 \wedge \Next \ (t_2 \wedge \Next \ (t_3 \wedge \Next \ (t_4 \wedge \Next \ t_5 \wedge s_4 )))))$, and $\phi_3 = \bigwedge_{i\in \{2,4,5\}} \Always \Event \, s_i$.

We have implemented the proposed solution in MATLAB, and we illustrate the resulting trace prefixes after 40 iterations in~Fig.~\ref{fig:example}.(B). It can be seen that the agents make progress towards satisfaction of their respective formulas. 
In the computation, the default values of planning horizons were $h=3$, and $H=5$. The latter value was sometimes too low to find a solution, thus, in several cases it has been extended as described in~\ref{sec:relax}. The maximum value needed in order to find a solution was $H=9$.
The sizes of the product automata handled in each iteration of the algorithm are depicted in Fig.~\ref{fig:sizes}. {In the centralized solution,} all three agents  belong to the dependency class, and hence, their synchronized product  transition system has $144^3 \approx 3$ million states. In contrast, in our solution, the decomposition into dependency classes is done locally, and at most two agents belong to the same dependency class at the time (in iterations 1-5, and 17-31), resulting into product system sizes in order of thousands states. When the agents are not dependent on each other within $h$ (in iterations 6-16, 32-40), the sizes of product systems are tens to hundreds states.

\begin{figure}[h!]
\begin{center}
\begin{tabular}{c}
{\hspace{-8cm}(A)}  \\ 
\scalebox{0.5}{
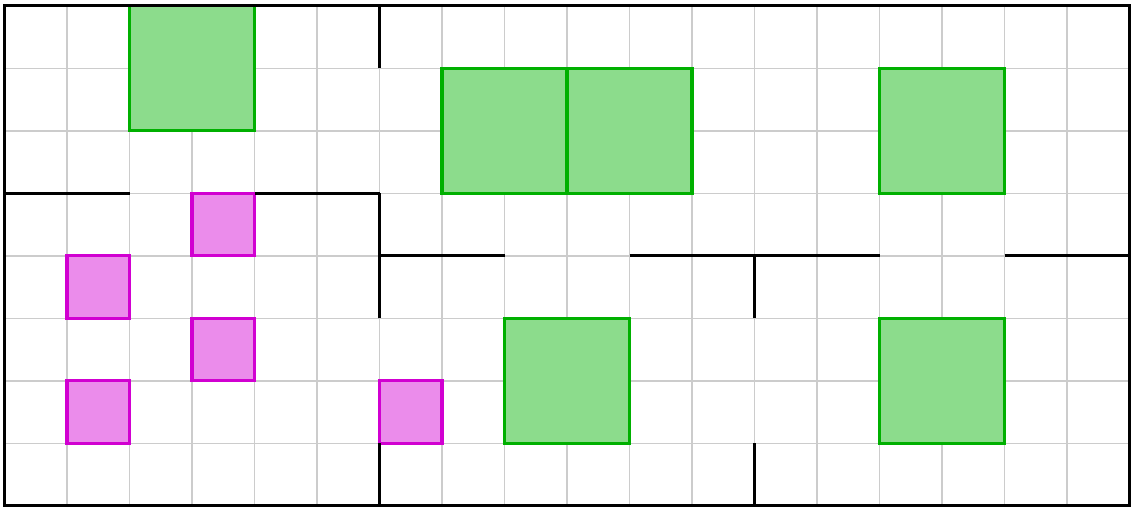 } \\
\hspace{-8cm}(B) \\
\scalebox{0.5}{
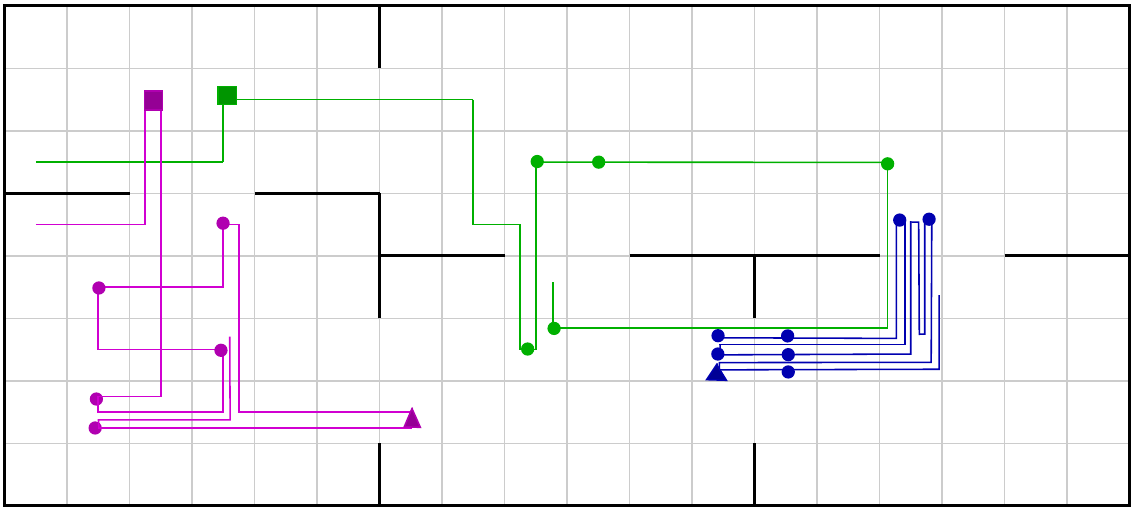 }

\end{tabular}
\end{center}
\caption{\footnotesize(A) An example of an environment partitioned into cells. The environment consists of rooms $R_1,\ldots, R_5$. Green regions are loading and unloading points for a heavy object $H$ and light objects $A$, $B$, $C$. Purple regions depict those where simple tasks $t_1,\ldots,t_5$ can be executed. (B) Traces of agent 1 (green), agent 2 (purple), and agent 3 (blue) after 40 iterations of Alg.~\ref{alg:main}. The initial position of the agents are in the bottom left corner of $R_1$, in the top left corner of $R_3$, and in the cell labeled with $s_4$, respectively. Services $l_H$ and $h_H$, and $t_5$ and $s_4$ are provided at the same time, illustrated as squares, and triangles, respectively. The rest of the provided services are depicted as circles.}
\label{fig:example}

\end{figure}

\begin{figure}[!h]
\begin{center}
\includegraphics[width=0.8\linewidth]{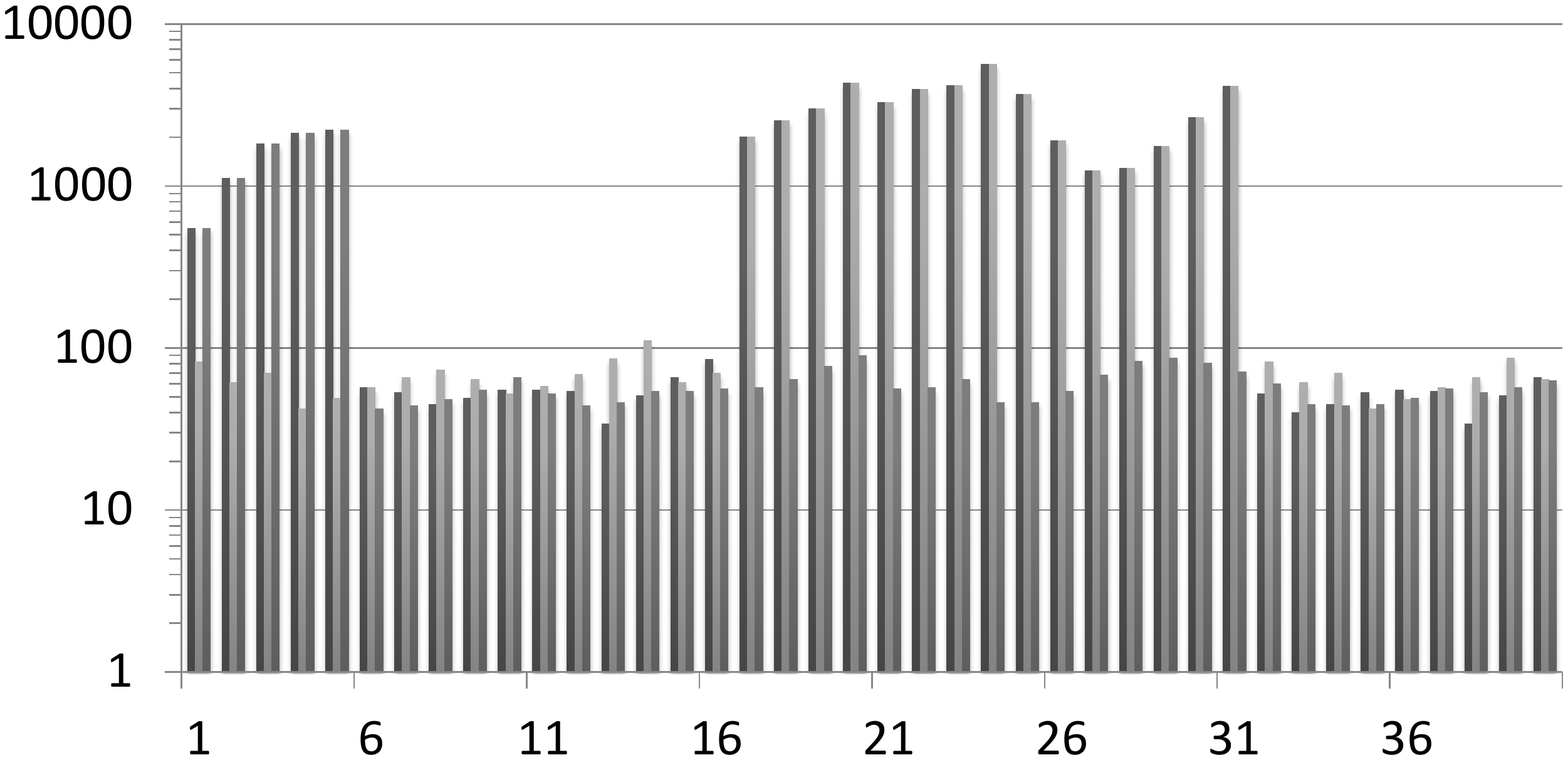}

\caption{\footnotesize The sizes of product automata in time (in logarithmic scale). The horizontal axis is labeled with the algorithm iteration number, the vertical one with the number of states of the product systems.}
\label{fig:sizes}
\end{center}

\end{figure}

\section{Summary and Future Work}
\label{sec:summary}
\new{We have proposed an automata-based receding horizon approach to solve the multi-agent planning problem from local LTL specifications. The solution decomposes the infinite horizon planning problem into a finite horizon planning problems that are solved iteratively. Such solution brings two major advantages over the offline, centralized solution: First, the limited planning horizon enables each agent to restrict its focus only on those agents, that are constrained by its formula within the limited horizon, not within the whole infinite horizon. Thus, we reach a partially decentralized solution. Second, we reduce the size of handled state space.

Future research directions include involving various optimality requirements. 
Another aspect that we would like to address is robustness to small perturbations; an offline planning procedure with deterministic transition systems is not suitable for such problems and the complexity of planning with non-deterministic system is unbearable.
}

\bibliographystyle{plain}
\bibliography{refer}

\end{document}